\documentclass[runningheads,10pt]{llncs}
\usepackage{llncsdoc}
\usepackage{times}
\usepackage[tbtags]{amsmath}
\usepackage{amssymb}
\usepackage{mathtools}
\usepackage{enumerate}
\usepackage{tikz}
\usepackage{url}
\usepackage{graphicx}
\usepackage{multirow}
\usepackage{nth}
\usepackage{longtable}
\usepackage{caption}
\usepackage{url}
\usetikzlibrary{trees,shapes,arrows.meta}
\usepackage{enumitem}
\setitemize{topsep=0.1em,parsep=0.1em,partopsep=0.1em}

\allowdisplaybreaks[4]
\newcommand{\shortsubsection}[1]{\par \noindent\textbf{#1}\  }

\newcommand\solidline[1][0.35cm]{\rule[0.5ex]{#1}{.4pt}}
\newcommand\dashedline{\mbox{
  \solidline[0.5mm]\hspace{0.8mm}\solidline[0.5mm]\hspace{0.8mm}\solidline[0.5mm]}}

\newcommand{\rul}{\ensuremath{\mathbf{r}}}
\newcommand{\heur}{\ensuremath{h}}

\begin{document}
\title{Exploiting Anti-monotonicity of Multi-label Evaluation Measures for Inducing Multi-label Rules \\ \normalsize \normalfont \ \\Preprint version. To be published in: Proceedings of the Pacific-Asia Conference on Knowledge Discovery and Data Mining (PAKDD) 2018}
\titlerunning{Exploiting Anti-monotonicity of Multi-label Evaluation Measures}
\author{Michael Rapp (\email{mrapp@ke.tu-darmstadt.de}) \\Eneldo Loza Menc\'ia (\email{eneldo@ke.tu-darmstadt.de}) \\Johannes F\"urnkranz (\email{fuernkranz@ke.tu-darmstadt.de})}
\authorrunning{Michael Rapp \and Eneldo Loza Menc\'ia \and Johannes F\"urnkranz}
\institute{Technische Universit\"at Darmstadt, Knowledge Engineering Group \\ Hochschulstrasse 10, D-64289 Darmstadt}
\maketitle

\begin{abstract}
  Exploiting dependencies between labels is considered to be crucial for multi-label classification. Rules are able to expose label dependencies such as implications, subsumptions or exclusions in a human-comprehensible and interpretable manner. However, the induction of rules with multiple labels in the head is particularly challenging, as the number of label combinations which must be taken into account for each rule grows exponentially with the number of available labels. To overcome this limitation, algorithms for exhaustive rule mining typically use properties such as anti-monotonicity or decomposability in order to prune the search space. In the present paper, we examine whether commonly used multi-label evaluation metrics satisfy these properties and therefore are suited to prune the search space for multi-label heads.
\end{abstract}

\section{Introduction}

Multi-label classification (MLC) is the task of learning a model for assigning a set of labels to unknown instances \cite{tsoumakas10MLoverview}. For example, newspaper articles can often be associated with multiple topics. This is in contrast to binary or multi-class classification, where single classes are predicted. As many studies show, MLC approaches that are able to take correlations between labels into account can be expected to achieve better predictive results (see \cite{demb12,menc15,tsoumakas10MLoverview}; and references therein).

In addition to statistical approaches that often rely on complex mathematical concepts, such as Bayesian or neural networks, rule learning algorithms have recently been proposed as an alternative, because rules are not only a natural and simple form to represent a learned model, but they are well suited for making discovered correlations between instance and label attributes explicit \cite{menc15}. Especially for safety-critical application domains, such as medicine, power systems, autonomous driving or financial markets, where hidden malfunctions could lead to life-threatening actions or economic loss, the possibility of interpreting, inspecting and verifying a classification model is essential (cf. e.g., \cite{kayande2009incorporating}). However, the algorithm of \cite{menc15}, which is based on the separate-and-conquer (SeCo) strategy, can only learn dependencies where the presence or absence of a single label depends on a subset of the instance's features. Especially co-occurrences of labels -- a common pattern in multi-label data -- are hence only representable by a combination of rules. Conversely, algorithms based on subgroup discovery were proposed which are able to find single rules that predict a subset of the possible labels \cite{robardet2016local}. However, this framework is limited in the sense that it relies on the adaptation of conventional rule learning heuristics for rating and selecting candidate rules and can thus not be easily adapted to a variety of different loss functions which are commonly used for evaluating multi-label predictions. Such an adaptation is not straight-forward, because it is not known whether these measures satisfy properties like anti-monotonicity that can ensure an efficient exploration of the search space of all possible rule heads -- despite the fact that it grows exponentially with the number of available labels.

Thus, the main contribution of this work (presented in Section~\ref{sec:properties}) is to formally define anti-monotonicity in the context of multi-label rules and to prove that selected multi-label metrics satisfy that property. Based on these findings, we present an algorithm that prunes searches for multi-label rules in Section~\ref{sec:algorithm}. Said algorithm is not meant to set new standards in terms of predictive performance, but to serve as a starting point for developing more enhanced approaches. Nevertheless, we evaluate that it is able to compete with different baselines in terms of predictive and -- more importantly -- computational performance in Section~\ref{sec:evaluation}.

\section{Preliminaries}
\label{sec:preliminaries}

The task of MLC is to associate an instance with one or several labels $\lambda_{i}$ out of a finite label space $\mathbb{L} = (\lambda_{i},...,\lambda_{n})$ with $n = |\mathbb{L}|$ being the number of available labels. An instance $X_{j}$ is typically represented in attribute-value form, i.e., it consists of a vector $X_{j} \coloneqq \langle v_{1},...,v_{l} \rangle \in \mathbb{D} = A_{1} \times ... \times A_{l}$ where $A_i$ is a numeric or nominal attribute. Each instance is mapped to a binary \emph{label vector} $Y_{j} \in \{0,1\}^n$ which specifies the labels that are associated with the example $X_{j}$. Consequently, the training data set of a MLC problem can be defined as a sequence of tuples $T \coloneqq \langle (X_{1},Y_{1}),...,(X_{m},Y_{m}) \rangle \subseteq \mathbb{D} \times \mathbb{L}$ with $m = |T|$. The model which is derived from a given multi-label data set can be viewed as a classifier function $g(.)$ mapping a single example $X$ to a prediction $\hat{Y} = g(X)$.

\subsection{Multi-label rule learning}
\label{section_multi_label_rules}

We are concerned with learning multi-label rules $\rul: \hat{Y} \leftarrow B$. The body $B$ may consist of several conditions, the examples that are covered by the rule have to satisfy. In this work only conjunctive, propositional rules are considered, i.e., each condition compares an attribute's value to a constant by either using equality (nominal attributes) or inequalities (numerical attributes). It is also possible to include label conditions in the body \cite{menc15,male97}. This allows to expose and distinct between \emph{unconditional} or \emph{global} dependencies and \emph{conditional} or \emph{local} dependencies \cite{demb12}.

The head $\hat{Y}$ consists of one or several label attributes ($\hat{y}_{i} = 0$ or $\hat{y}_{i} = 1$) which specify the absence or presence of the corresponding label $\hat{y}_{i}$. Rules that contain a single label attribute in their head are referred to as \emph{single-label head rules}, whereas \emph{multi-label head rules} may contain several label attributes in their head.

A predicted label vector $\hat{Y}$ may have different semantics. We differentiate between \emph{full predictions} and \emph{partial predictions}.
\begin{itemize}
\item \textbf{Full predictions:} Each rule predicts a full label vector, i.e., if a label attribute $\hat{y}_{i}$ is not contained in the head, the absence of the corresponding label $\lambda_{i}$ is predicted.
\item \textbf{Partial predictions:} Each rule predicts the presence or absence of the label only for a subset of the possible labels. For the remaining labels the rule does not make a prediction (but other rules might).
\end{itemize}
We believe that partial predictions have several conceptual and practical advantages and therefore we focus on that particular strategy throughout the remainder of this work.

\subsection{Bipartition evaluation functions}

To evaluate the quality of multi-label predictions, we use bipartition evaluation measures (cf.\ \cite{tsoumakas10MLoverview}) which are based on evaluating differences between true (\emph{ground truth}) and predicted label vectors. They  can be considered as functions of two-dimensional \emph{label confusion matrices} which represent the \emph{true positive} ($TP$), \emph{false positive} ($FP$), \emph{true negative} ($TN$) and \emph{false negative} ($FN$) label predictions. For a given example $X_{j}$ and a label $y_{i}$ the elements of an atomic confusion matrix $C_{i}^{j}$ are computed as
\begin{equation}
  \label{equation_confusion_matrix}
  C_i^j = \left(\begin{matrix}
    TP_i^j & FP_i^j \\
    FN_i^j & TN_i^j
  \end{matrix}\right)
= \left( \begin{matrix}
  y_{i}^j  \hat{y}_{i}^j &&
(1-y_{i}^j)  \hat{y}_{i}^j \\
 (1-y_{i}^j)  (1-\hat{y}_{i}^j) &&
 y_{i}^j  (1-\hat{y}_{i}^j)
\end{matrix}\right)
\end{equation}
where the variables $y_{i}^j$ and $\hat{y}_{i}^j$ denote the absence (0) or presence (1) of label $\lambda_{i}$ of example $X_j$ according to the ground truth or the predicted label vector, respectively.

Note that for candidate rule selection we assess $TP$, $FP$, $TN$, and $FN$ differently. To ensure that absent and present labels have the same impact on the performance of a rule, we always count correctly predicted labels as $TP$ and incorrect predictions as $FP$, respectively. Labels for which no prediction is made are counted as $TN$ if they are absent, or as $FN$ if they are present.

\subsubsection{Multi-label evaluation functions}

In the following some of the most common bipartition metrics $\delta(C)$ used for MLC are presented (cf., e.g., \cite{tsoumakas10MLoverview}). They are surjections $\mathbb{N}^{2x2} \rightarrow \mathbb{R}$ mapping a confusion matrix $C$ to a heuristic value $h \in \left[0,1\right]$. Predictions that reach a greater heuristic value outperform those with smaller values.
\begin{itemize}
  \item \textbf{Precision:} Percentage of correct predictions among all predicted labels.
  \begin{equation}
  \label{equation_precision}\footnotesize
  \delta_{prec}(C) \coloneqq \frac{TP}{TP+FP}
  \end{equation}
  \item \textbf{Hamming accuracy:} Percentage of correctly predicted present and absent labels among all labels.
  \begin{equation}
  \label{equation_hamming_accuracy}\footnotesize
  \delta_{hamm}(C) \coloneqq \frac{TP+TN}{TP+FP+TN+FN}
  \end{equation}
  \item \textbf{F-measure:} Weighted harmonic mean of precision and recall. If $\beta < 1$, precision has a greater impact. If $\beta > 1$, the F-measure becomes more recall-oriented.
  \begin{equation}
  \label{equation_f_measure}\footnotesize
  \delta_{F}(C) \coloneqq \frac{\beta^2 + 1}{\frac{\beta^2}{\delta_{rec}(C)} + \frac{1}{\delta_{prec}(C)}} \text{ , with } \delta_{rec}(C) = \frac{TP}{TP + FN} \text{ and } \beta \in \left[0,\infty\right]
  \end{equation}
  \item \textbf{Subset accuracy:} Percentage of perfectly predicted label vectors among all examples. Per definition, it is always calculated using example-based averaging.
  \begin{equation}
  \label{equation_subset_accuracy}\footnotesize
  \delta_{acc}(C) \coloneqq \frac{1}{m} \sum \limits_{j} \left[ Y_{j} = \hat{Y}_{j} \right] \text{ , with } [x] = \begin{cases}
    1, & \text{if $x$ is true} \\
    0, & \text{otherwise}
  \end{cases}
  \end{equation}
\end{itemize}

\subsubsection{Aggregation and averaging}

When evaluating multi-label predictions which have been made for $m$ examples with $n$ labels one has to deal with the question of how to aggregate the resulting $m \cdot n$ atomic confusion matrices. Essentially, there are four possible averaging strategies -- either \emph{(label- and example-based) micro-averaging}, \emph{label-based (macro-)averaging}, \emph{example-based (macro-) averaging} or \emph{(label- and example-based) macro-averaging}. Due to the space limitations, we restrict our analysis to the most popular aggregation strategy employed in the literature, namely \emph{micro-averaging}. This particular averaging strategy is formally defined as
\begin{equation}
\label{equation_micro_averaging}
\footnotesize
\delta(C) = \delta \left( \sum \nolimits_{j}^{} \sum \nolimits_{i}^{} C_{i}^j \right) \equiv \delta \left( \sum \nolimits_{i}^{} \sum \nolimits_{j}^{} C_{i}^j \right)
\end{equation}
where the $\sum$ operator denotes the cell-wise addition of confusion matrices.

\subsubsection{Relation to conventional association rule discovery}
\label{sec:relation}

To illustrate the difference between measures used in association rule discovery and in multi-label rule learning, assume that the rule $\lambda_1, \lambda_2 \leftarrow B$ covers three examples $(X_1,\{\lambda_2\})$, $(X_2,\{\lambda_1,\lambda_2\})$ and $(X_3,\{\lambda_1\})$. In conventional association rule discovery the head is considered to be satisfied for one of the three covered examples ($X_2$), yielding a precision/confidence value of $\frac{1}{3}$. This essentially corresponds to subset accuracy. On the other hand, micro-averaged precision would correspond to the fraction of $4$ correctly predicted labels among $6$ predictions, yielding a value of $\frac{2}{3}$.

\section{Properties of multi-label evaluation measures}
\label{sec:properties}

To induce multi-label head rules, we need to find the multi-label head $\hat{Y}$ which reaches the best possible performance
\begin{equation}
h_{max} = \max_{\hat{Y}} \heur(\rul) = \max_{\hat{Y}} \heur(\hat{Y} \leftarrow B)
\end{equation}
given an evaluation function $\heur(.)$ and a body $B$. In this section we consider rule evaluation functions that are based on micro-averaged atomic confusion matrices in a partial prediction setting, i.e., $\heur(\rul) = \delta(C)$ where $\delta(C)$ is defined as in \eqref{equation_micro_averaging}.

Due to the exponential complexity of an exhaustive search, it is crucial to prune the search for the best multi-label head by leaving out unpromising label combinations. The first property which can be exploited for pruning searches -- while still being able to find the best solution -- is \emph{anti-monotonicity}.

\begin{definition}[Anti-monotonicity]
\label{definition_anti_monotonicity}
Let $\hat{Y}_{p} \leftarrow B$ and $\hat{Y}_{s} \leftarrow B$ denote two multi-label head rules consisting of body $B$ and heads $\hat{Y}_{p}$, respectively $\hat{Y}_{s}$. It is further assumed that $\hat{Y}_{p} \subset \hat{Y}_{s}$. A multi-label evaluation function $\heur$ is \emph{anti-monotonic} if the following condition is met, i.e., if no head $Y_{a}$ that results from adding additional labels to $Y_{s}$ may result in $h_{max}$ being reached:
\[
\hat{Y}_{p} \subset \hat{Y}_{s} \wedge \heur(\hat{Y}_{s} \leftarrow B) < \heur(\hat{Y}_{p} \leftarrow B) \Longrightarrow \heur(\hat{Y}_{a} \leftarrow B) < h_{max} \text{ , } \forall \hat{Y}_{a}: \hat{Y}_{s} \subset \hat{Y}_{a}
\]
\end{definition}

In addition to the adaptation of anti-monotonicity in Definition~\ref{definition_anti_monotonicity}, we propose \emph{decomposability} as a stronger criterion. It comes at linear costs, as the best possible head can be deduced from considering each available label separately. Due to its restrictiveness, if Definition~\ref{definition_decomposability} is met, Definition~\ref{definition_anti_monotonicity} is implied to be met as well.

\begin{definition}[Decomposability]
\label{definition_decomposability}
A multi-label evaluation function $\heur$ is \emph{decomposable} if the following conditions are met:
\begin{enumerate}[label=\emph{\roman*})]
  \item If the multi-label head rule $\hat{Y} \leftarrow B$ contains a label attribute $\hat{y}_{i} \in \hat{Y}$ for which the corresponding single-label head rule $\hat{y}_{i} \leftarrow B$ does not reach $h_{max}$, the multi-label head rule cannot reach that performance either (and vice versa).
\[
\exists i \left( \hat{y}_{i} \in \hat{Y} \wedge \heur(\hat{y}_{i} \leftarrow B) < h_{max} \right) \Longleftrightarrow \heur(\hat{Y} \leftarrow B) < h_{max}
\]
  \item If all single label head rules $\hat{y}_{i} \leftarrow B$ which correspond to the label attributes of the multi-label head $\hat{Y}$ reach $h_{max}$, the multi-label head rule $\hat{Y} \leftarrow B$ reaches that performance as well (and vice versa).
\[
\heur(\hat{y}_{i} \leftarrow B) = h_{max} \text{ , } \forall \hat{y}_{i} \left( \hat{y}_{i} \in \hat{Y} \right) \Longleftrightarrow \heur(\hat{Y} \leftarrow B) = h_{max}
\]
\end{enumerate}
\end{definition}
In the following we examine selected multi-label metrics in terms of decomposability and anti-monotonicity to reveal whether they satisfy these properties when making partial predictions (cf. Section \ref{section_multi_label_rules}).
\begin{theorem}
\label{theorem_precision_mm}
Micro-averaged precision is decomposable.
\end{theorem}
\begin{proof}
We rewrite the performance calculation for a multi-label head rule $\rul: \hat{Y} \leftarrow B$ with $\heur(\rul) = h_{max}$ using the fact that the single label head rules $\rul_i: \hat{y}_{i} \leftarrow B$ with $\hat{y}_{i} \in \hat{Y}$ share the same body $B$ and therefore cover the same number of examples $|C|$.
\begin{equation}
\footnotesize
\label{equation_precision_mm}
\begin{split}
\heur(\rul) = & \frac{\sum \limits_{\hat{y}_{i} \in \hat{Y}}^{} \sum \limits_{j}^{} TP_{i}^{j}}{\sum \limits_{\hat{y}_{i} \in \hat{Y}}^{} \sum \limits_{j}^{} p_{i}^{j}} \text{ , with } p_{i}^{j} = TP_{i}^{j} + FP_{i}^{j} \text{ and } \sum \limits_{j}^{} p_{i}^{j} = |C| \text{ , } \forall i \\
= & \frac{\sum \limits_{\hat{y}_{i} \in \hat{Y}}^{} \sum \limits_{j}^{} TP_{i}^{j}}{|\hat{Y}| \cdot |C|} = \frac{1}{|\hat{Y}|} \sum \limits_{\hat{y}_{i} \in \hat{Y}}^{} \frac{\sum \limits_{j}^{} TP_{i}^{j}}{|C|} \equiv \frac{1}{|\hat{Y}|} \sum \limits_{\hat{y}_{i} \in \hat{Y}}^{} \heur(\rul_i)
\end{split}
\end{equation}
Thus, the micro-averaged precision for $\rul$ corresponds to the average of the micro-averaged precision of the single-label head rules $\rul_i$. As we assume that $\heur(\rul)$ is maximal, it follows that $\heur(\rul) = \heur(\rul_i)$ for all single-label head rules $\rul_i$.
\end{proof}
\begin{theorem}
Micro-averaged Hamming accuracy is decomposable.
\end{theorem}
\begin{proof}
Similar to \eqref{equation_precision_mm}, we rewrite the micro-averaged Hamming accuracy of a multi-label head rule $\rul: \hat{Y} \leftarrow B$ with $\heur(\rul) = h_{max}$ in terms of averaging the performance of single-label head rules $\rul_i: \hat{y}_{i} \leftarrow B$. This is possible as the performance for each label $\hat{y}_{i}$ calculates as the percentage of $TP$ and $TN$ among all $m$ labels. For reasons of simplicity, we use the abbreviations $P_{i}^{j} = TP_{i}^{j} + FN_{i}^{j}$ and $N_{i}^{j} = FP_{i}^{j} + TN_{i}^{j}$.
\begin{equation}
\footnotesize
\label{equation_hamming_accuracy_mm}
\begin{split}
\heur(\rul) = & \frac{\sum \limits_{\hat{y}_{i} \in \hat{Y}} \sum \limits_{j} \left( TP_{i}^{j} + TN_{i}^{j} \right)}{\sum \limits_{\hat{y}_{i} \in \hat{Y}} \sum \limits_{j} \left( P_{i}^{j} + N_{i}^{j} \right)} \text{ , with } \sum \limits_{j} \left( P_{i}^{j} + N_{i}^{j} \right) = m \text{ , } \forall i \\
= & \frac{\sum \limits_{\hat{y}_{i} \in \hat{Y}} \sum \limits_{j} \left( TP_{i}^{j} + TN_{i}^{j} \right)}{|\hat{Y}| \cdot m} = \frac{1}{|\hat{Y}|} \sum \limits_{\hat{y}_{i} \in \hat{Y}} \frac{\sum \limits_{j} \left( TP_{i}^{j} + TN_{i}^{j} \right)}{m} \equiv \frac{1}{|\hat{Y}|} \sum \limits_{\hat{y}_{i} \in \hat{Y}} \heur(\rul_i)
\end{split}
\end{equation}
\end{proof}
\begin{theorem}
Subset accuracy is anti-monotonic.
\end{theorem}
\begin{proof}
In accordance with Definition~\ref{definition_anti_monotonicity}, two multi-label head rules $\hat{Y}_{p} \leftarrow B$ and \linebreak $\hat{Y}_{s} \leftarrow B$, for whose heads the subset relationship $\hat{Y}_{p} \subset \hat{Y}_{s}$ holds, take part in equation \eqref{equation_subset_accuracy_am}. The subscript notation $\left.x\right|_{\hat{Y}}$ is used to denote that a left-hand expression $x$ should be evaluated using the rule $\hat{Y} \leftarrow B$. The proof is based on writing subset accuracy in terms of $TP$ and $TN$ (cf. line 2).
\begin{align}
\footnotesize
\label{equation_subset_accuracy_am}
\begin{split}
& \hat{Y}_{p} \subset \hat{Y}_{s} \wedge \heur(\hat{Y}_{s} \leftarrow B) < \heur(\hat{Y}_{p} \leftarrow B) \\
\Rightarrow & \left.\frac{1}{m} \sum \limits_{j} \left[ \sum \limits_{\hat{y}_{i} \in \hat{Y}} \left( TP_{i}^{j} + TN_{i}^{j} \right) = |\hat{Y}| \right] \right|_{\hat{Y}_{s}} \hspace{-1em} < \left.\frac{1}{m} \sum \limits_{j} \left[ \sum \limits_{\hat{y}_{i} \in \hat{Y}} \left( TP_{i}^{j} + TN_{i}^{j} \right) = |\hat{Y}| \right] \right|_{\hat{Y}_{p}} \hspace{-1em} \leq h_{max} \\
\Rightarrow & \exists j \left( 0 = \left.\left[ \sum \limits_{\hat{y}_{i} \in \hat{Y}} \left( TP_{i}^{j} + TN_{i}^{j} \right) = |\hat{Y}| \right] \right|_{\hat{Y}_{s}} < \left.\left[ \sum \limits_{\hat{y}_{i} \in \hat{Y}} \left( TP_{i}^{j} + TN_{i}^{j} \right) = |\hat{Y}| \right] \right|_{\hat{Y}_{p}} = 1 \right) \\
\Rightarrow & \exists \hat{y}_{i} \exists j \left( \hat{y}_{i} \in \hat{Y}_{s} \wedge \left.\left(TP_{i}^{j} + TN_{i}^{j} \right) < |\hat{Y}| \right|_{\hat{Y}_{s}} \right) \\
\Rightarrow & \exists \hat{y}_{i} \exists j \left( \hat{y}_{i} \in \hat{Y}_{a} \wedge \left.\left(TP_{i}^{j} + TN_{i}^{j} \right) < |\hat{Y}| \right|_{\hat{Y}_{a}} \right) \text{ , } \forall \hat{Y}_{a} \left( \hat{Y}_{s} \subset \hat{Y}_{a} \right) \\
\Rightarrow & \exists j \left( \left.\left[ \sum \limits_{\hat{y}_{i} \in \hat{Y}} \left( TP_{i}^{j} + TN_{i}^{j} \right) = |\hat{Y}| \right] \right|_{\hat{Y}_{a}} = 0 \right) \text{ , } \forall \hat{Y}_{a} \left( \hat{Y}_{s} \subset \hat{Y}_{a} \right) \\
\Rightarrow & \frac{1}{m} \sum \limits_{j} \left.\left[ \sum \limits_{\hat{y}_{i} \in \hat{Y}} \left( TP_{i}^{j} + TN_{i}^{j} \right) = |\hat{Y}| \right] \right|_{\hat{Y}_{a}} < h_{max} \text{ , } \forall \hat{Y}_{a} \left( \hat{Y}_{s} \subset \hat{Y}_{a} \right) \\
\equiv & \heur(\hat{Y}_{a} \leftarrow B) < h_{max} \text{ , } \forall \hat{Y}_{a} \left( \hat{Y}_{s} \subset \hat{Y}_{a} \right)
\end{split}
\end{align}
In \eqref{equation_subset_accuracy_am} it is concluded that when using the rule $\hat{Y}_{s} \leftarrow B$ the performance for at least one example $Y_{j}$ is less than when using the rule $\hat{Y}_{p} \leftarrow B$. Due to the definition of subset accuracy, the performance for that example must be 0 in the first case and 1 in the latter (cf. line 3). As the performance only evaluates to 0 if at least one label is predicted incorrectly, the head $\hat{Y}_{p}$ must contain a label attribute $\hat{y}_{i}$ which predicts the corresponding label incorrectly (cf. line 4). When adding additional label attributes the prediction for that label will still be incorrect (cf. line 5). Therefore, for all multi-label head rules $\hat{Y}_{a} \leftarrow B$ which result from adding additional label attributes to the head $\hat{Y}_{s}$ the performance for the example $Y_{j}$ evaluates to 0 (cf. line 6). Consequently, none of them can reach the overall performance of $\hat{Y}_{p} \leftarrow B$, nor $h_{max}$ (cf. line 7 and 8).
\end{proof}
\begin{lemma}
\label{lemma_recall_mm}
Micro-averaged recall is decomposable.
\end{lemma}
\begin{proof}
The \emph{mediant} of fractions $\frac{a_{1}}{b_{1}},...,\frac{a_{n}}{b_{n}}$ is defined as $\frac{a_{1} + ... + a_{n}}{b_{1} + ... + b_{n}}$. The micro-averaged recall of a multi-label head rule $\rul: \hat{Y} \leftarrow B$ is the mediant of the performances which are obtained for corresponding single-label head rules $\rul_i: \hat{y}_{i} \leftarrow B$ with $\hat{y}_{i} \in \hat{Y}$ according to the recall metric.
\begin{equation}
\footnotesize
\label{equation_recall}
\heur(\rul) = \frac{\sum \limits_{\hat{y}_{i} \in \hat{Y}} \sum \limits_{j} TP_{i}^{j}}{\sum \limits_{\hat{y}_{i} \in \hat{Y}} \sum \limits_{j} \left( TP_{i}^{j} + FN_{i}^{j} \right)}
\end{equation}
The \emph{mediant inequality} states that the mediant strictly lies between the fractions it is calculated from, i.e., that $min \left( \frac{a_{1}}{b_{1}},...,\frac{a_{n}}{b_{n}}\right) \leq \frac{a_{1} + ... + a_{n}}{b_{1} + ... + b_{n}} \leq max \left( \frac{a_{1}}{b_{1}},...,\frac{a_{n}}{b_{n}} \right)$. This is in accordance with Definition~\ref{definition_decomposability}.
\end{proof}
\begin{theorem}
Micro-averaged F-measure is decomposable.
\end{theorem}
\begin{proof}
Micro-averaged F-measure calculates as the (weighted) harmonic mean $H(.)$ of micro-averaged precision and recall. This proof is based on the finding that both of these metrics fulfill the properties of decomposability (cf. Theorem~\ref{theorem_precision_mm} and Lemma~\ref{lemma_recall_mm}). As multiple metrics take part in the proof, we use a superscript notation to distinguish between the best possible performances according to different metrics, e.g., $h_{max}^{F}$ in case of the F-measure. Furthermore, we exploit the inequality $h_{max}^{F} \leq max \left( h_{max}^{rec}, h_{max}^{prec} \right)$.
\begin{equation}
\label{equation_f_measure_mm_1}\footnotesize
\begin{split}
& \exists i \left( \hat{y}_{i} \in \hat{Y} \wedge \heur_{F}(\hat{y}_{i} \leftarrow B) < h_{max}^{F} \leq h_{max}^{rec} \right) \\
\equiv & \exists i \left( \hat{y}_{i} \in \hat{Y} \wedge H \left( \heur_{rec}(\hat{y}_{i} \leftarrow B), \heur_{prec}(\hat{y}_{i} \leftarrow B) \right) < h_{max}^{rec} \right) \\
\Rightarrow & \exists i \left( \hat{y}_{i} \in \hat{Y} \wedge \left( \heur_{rec}(\hat{y}_{i} \leftarrow B) < h_{max}^{rec} \wedge \heur_{prec}(\hat{y}_{i} \leftarrow B) < h_{max}^{rec} \right) \right. \\
& \vee \left. \left( \heur_{prec}(\hat{y}_{i} \leftarrow B) < h_{max}^{rec} \wedge \heur_{rec}(\hat{y}_{i} \leftarrow B) \leq h_{max}^{rec} \right) \right) \\
\Rightarrow & \left( \heur_{rec}(\hat{Y} \leftarrow B) < h_{max}^{rec} \wedge \heur_{prec}(\hat{Y} \leftarrow B) < h_{max}^{rec} \right) \\
& \vee \left( \heur_{prec}(\hat{Y} \leftarrow B) < h_{max}^{rec} \wedge \heur_{rec}(\hat{Y} \leftarrow B) \leq h_{max}^{rec} \right) \\
\Rightarrow & H \left( \heur_{rec}(\hat{Y} \leftarrow B), \heur_{prec}(\hat{Y} \leftarrow B) \right) < h_{max}^{F} \leq h_{max}^{rec} \\
\equiv & \heur_{F}(\hat{Y} \leftarrow B) < h_{max}^{F}
\end{split}
\end{equation}
In \eqref{equation_f_measure_mm_1} the first property of Definition~\ref{definition_decomposability} is proved. As the premise of the proof, we assume w.l.o.g. that the best possible performance according to the recall metric is equal to or greater than the best performance according to precision, i.e., that the relation $h_{max}^{rec} \geq h_{max}^{prec}$ holds. We further assume that the F-measure of a single-label head rule $\rul_i: \hat{y}_{i} \leftarrow B$ is less than the best possible performance $h_{max}$ (cf. line 1 and 2). When rewriting the F-measure in terms of the harmonic mean of precision and recall, it follows that either recall or precision of $r_i$ must be less than $h_{max}^{F}$, respectively $h_{max}^{rec}$. Due to the premise of the proof, $h_{max}^{rec}$ can be considered as an upper limit for both recall and precision (cf. line 3). Furthermore, because precision and recall are decomposable, the multi-label head rule $\rul: \hat{Y} \leftarrow B$ with $\hat{y}_{i} \in \hat{Y}$ cannot outperform $h_{max}^{F}$ (cf. lines 5, 7 and 8). In order to prove the second property of decomposability to be met, the derivation in \eqref{equation_f_measure_mm_2} uses a similar approach as in \eqref{equation_f_measure_mm_1}. However, it is not based on its premise.
\begin{equation}
\label{equation_f_measure_mm_2}\footnotesize
\begin{split}
& \heur_{F}(\hat{y}_{i} \leftarrow B) = h_{max}^{F} \text{ , } \forall \hat{y}_{i} \left( \hat{y}_{i} \in \hat{Y} \right) \\
\equiv & H \left( \heur_{rec}(\hat{y}_{i} \leftarrow B), \heur_{prec}(\hat{y}_{i} \leftarrow B) \right) = h_{max}^{F} \text{ , } \forall \hat{y}_{i} \left( \hat{y}_{i} \in \hat{Y} \right) \\
\Longrightarrow & \heur_{rec}(\hat{y}_{i} \leftarrow B) = \heur_{prec}(\hat{y}_{i} \leftarrow B) = h_{max}^{F} \text{ , } \forall \hat{y}_{i} \left( \hat{y}_{i} \in \hat{Y} \right) \\
\Longrightarrow & \heur_{rec}(\hat{Y} \leftarrow B) = \heur_{prec}(\hat{Y} \leftarrow B) = h_{max}^{F} \\
\Longrightarrow & H \left( \heur_{rec}(\hat{Y} \leftarrow B), \heur_{prec}(\hat{Y} \leftarrow B) \right) = h_{max}^{F} \\
\equiv & \heur_{F}(\hat{Y} \leftarrow B) = h_{max}^{F}
\end{split}
\end{equation}
\end{proof}

\section{Algorithm for learning multi-label head rules}
\label{sec:algorithm}

To evaluate the utility of these properties, we implemented a multi-label rule learning algorithm based on the SeCo algorithm for learning single-label head rules by Loza Menc{\'\i}a an Janssen 
\cite{menc15}. Both algorithms share a common structure where new rules are induced iteratively and the examples they cover are removed from the training data set if enough of their labels are predicted by already learned rules. The rule induction process continues until only few training examples are left. To classify test examples, the learned rules are applied in the order of their induction. If a rule fires, the labels in its head are applied unless they were already set by a previous rule.

For learning new multi-label rules, our algorithm performs a top-down greedy search, starting with the most general rule. By adding additional conditions to the rule's body it can successively be specialized, resulting in less examples being covered. Potential conditions result from the values of nominal attributes or from averaging two adjacent values of the sorted examples in case of numerical attributes. Whenever a new condition is added, a corresponding single- or multi-label head that predicts the labels of the covered examples as accurate as possible must be found.

\shortsubsection{Evaluating possible multi-label heads}
To find the best head for a given body different label combinations must be evaluated by calculating a score based on the used averaging and evaluation strategy. The algorithm performs a breadth-first search by recursively adding additional label attributes to the (initially empty) head and keeps track of the best rated head. Instead of performing an exhaustive search, the search space is pruned according to the findings in Section~\ref{sec:properties}. When pruning according to anti-monotonicity unnecessary evaluations of label combinations are omitted in two ways: On the one hand, if adding a label attribute causes the performance to decrease, the recursion is not continued at deeper levels of the currently searched subtree. On the other hand, the algorithm keeps track of already evaluated or pruned heads and prevents these heads from being evaluated in later iterations. When a decomposable evaluation metric is used no deep searches through the label space must be performed. Instead, all possible single-label heads are evaluated in order to identify those that reach the highest score and merge them into one multi-label head rule.

\tikzstyle{level 1}=[level distance=1.3cm, sibling distance=3.8cm]
\tikzstyle{level 2}=[level distance=1.7cm, sibling distance=2.2cm]
\tikzstyle{level 3}=[level distance=1.9cm, sibling distance=2.1cm]
\tikzstyle{level 4}=[level distance=1.9cm]
\tikzstyle{bag} = [text width=5.0em, text centered]
\tikzstyle{circled-bag} = [bag, circle, dashed, draw=black, inner sep=0pt]
\tikzstyle{arrow} = [-{Latex[scale=1.5]}]
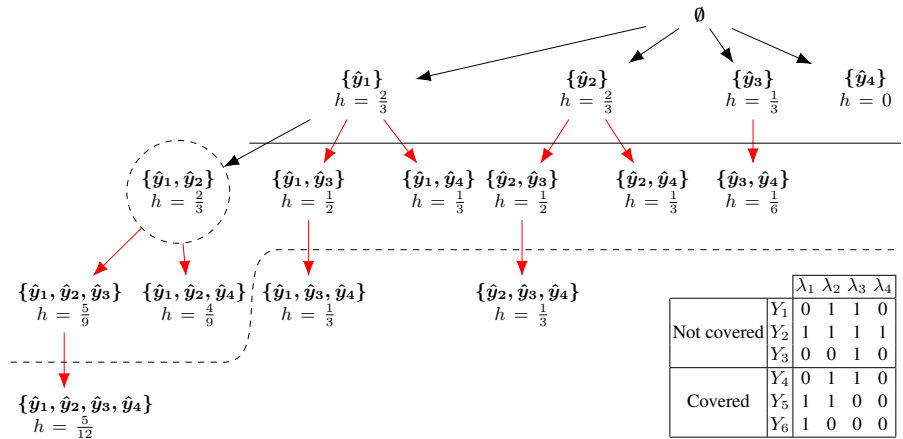
\begin{figure}[h]
\centering
\resizebox{\textwidth}{!}{
\begin{tikzpicture}
\node[bag] {$\boldsymbol{\emptyset}$}
  child[black,growth parent anchor={west}] {
    node(a2)[bag,black] {$\boldsymbol{\{\hat{y}_{1}\}}$ \\ $h=\frac{2}{3}$}
    child[black] {
      node[circled-bag,black] {$\boldsymbol{\{\hat{y}_{1},\hat{y}_{2}\}}$ \\ $h=\frac{2}{3}$}
      child[red,growth parent anchor={center}] {
        node(a)[bag,black] {$\boldsymbol{\{\hat{y}_{1},\hat{y}_{2},\hat{y}_{3}\}}$ \\ $h=\frac{5}{9}$}
          child[red] {
            node[bag,black] {$\boldsymbol{\{\hat{y}_{1},\hat{y}_{2},\hat{y}_{3},\hat{y}_{4}\}}$ \\ $h=\frac{5}{12}$}
            edge from parent[arrow]
          }
        edge from parent[arrow]
      }
      child[red] {
        node(b)[bag,black] {$\boldsymbol{\{\hat{y}_{1},\hat{y}_{2},\hat{y}_{4}\}}$ \\ $h=\frac{4}{9}$}
        edge from parent[arrow]
      }
      edge from parent[arrow]
    }
    child[red,growth parent anchor={center}] {
      node(c)[bag,black] {$\boldsymbol{\{\hat{y}_{1},\hat{y}_{3}\}}$ \\ $h=\frac{1}{2}$}
      child[red] {
        node[bag,black] {$\boldsymbol{\{\hat{y}_{1},\hat{y}_{3},\hat{y}_{4}\}}$ \\ $h=\frac{1}{3}$}
        edge from parent[arrow]
      }
      edge from parent[arrow]
    }
    child[red] {
      node[bag,black] {$\boldsymbol{\{\hat{y}_{1},\hat{y}_{4}\}}$ \\ $h=\frac{1}{3}$}
      edge from parent[arrow]
    }
    edge from parent[arrow]
  }
  child[black] {
    node[bag,black] {$\boldsymbol{\{\hat{y}_{2}\}}$ \\ $h=\frac{2}{3}$}
    child[red] {
      node[bag,black] {$\boldsymbol{\{\hat{y}_{2},\hat{y}_{3}\}}$ \\ $h=\frac{1}{2}$}
        child[red] {
          node[bag,black] {$\boldsymbol{\{\hat{y}_{2},\hat{y}_{3},\hat{y}_{4}\}}$ \\ $h=\frac{1}{3}$}
          edge from parent[arrow]
        }
      edge from parent[arrow]
    }
    child[red] {
      node(d)[bag,black] {$\boldsymbol{\{\hat{y}_{2},\hat{y}_{4}\}}$ \\ $h=\frac{1}{3}$}
      edge from parent[arrow]
    }
    edge from parent[arrow]
  }
  child[black,xshift=-1.0cm] {
    node[bag,black] {$\boldsymbol{\{\hat{y}_{3}\}}$ \\ $h=\frac{1}{3}$}
    child[red] {
      node(e)[bag,black] {$\boldsymbol{\{\hat{y}_{3},\hat{y}_{4}\}}$ \\ $h=\frac{1}{6}$}
      edge from parent[arrow]
    }
    edge from parent[arrow]
  }
  child[black,xshift=-2.9cm,yshift=+0.06cm] {
    node(b2)[bag,black,yshift=+0.0cm] {$\boldsymbol{\{\hat{y}_{4}\}}$ \\ $h=0$}
    edge from parent[arrow]
  };
  \draw[dashed] ([yshift=-0.5cm] a.south west) to ([xshift=0.6cm, yshift=-0.5cm] b.south) to[out=0,in=180] ([xshift=-0.5cm, yshift=-0.5cm] c.south) to ([yshift=-0.5cm] d.south) to ([xshift=1.5cm, yshift=-0.5cm] e.south east);
  \draw[solid] ([xshift=-1.0cm, yshift=-0.4cm] a2.south west) to ([xshift=+1.5cm, yshift=-0.4cm] b2.south west |-  a2.south east);

  \node(table) at ([xshift=-0.4cm,yshift=-2.3cm] e.south east) {
    \begin{tabular}{c c|c c c c|}
    \cline{3-6}
    & & $\lambda_{1}$ & $\lambda_{2}$ & $\lambda_{3}$ & $\lambda_{4}$ \\
    \hline
    \multicolumn{1}{|c|}{\multirow{3}{*}{Not covered}} & $Y_{1}$ & 0 & 1 & 1 & 0 \\
    \multicolumn{1}{|c|}{} & $Y_{2}$ & 1 & 1 & 1 & 1 \\
    \multicolumn{1}{|c|}{} & $Y_{3}$ & 0 & 0 & 1 & 0 \\
    \hline
    \multicolumn{1}{|c|}{\multirow{3}{*}{Covered}} & $Y_{4}$ & 0 & 1 & 1 & 0 \\
    \multicolumn{1}{|c|}{} & $Y_{5}$ & 1 & 1 & 0 & 0 \\
    \multicolumn{1}{|c|}{} & $Y_{6}$ & 1 & 0 & 0 & 0 \\
    \hline
    \end{tabular}
  };
\end{tikzpicture}
}
\caption{Search through the label space $\mathbb{L} = (\lambda_{1},\lambda_{2},\lambda_{3},\lambda_{4})$ using micro-averaged precision of partial predictions. The examples corresponding to label sets $Y_{4},Y_{5},Y_{6}$ are assumed to be covered, whereas those of $Y_{1},Y_{2},Y_{3}$ are not. The dashed line (\protect\dashedline) indicates label combinations that can be pruned with anti-monotonicity, the solid line (\protect\solidline) corresponds to decomposability.}
\label{figure_example}
\end{figure}
\pagebreak
Fig.~\ref{figure_example} illustrates how the algorithm prunes a search through the label space using anti-monotonicity and decomposability. The nodes of the given search tree correspond to the evaluations of label combinations, resulting in heuristic values $\heur$. The edges correspond to adding an additional label to the head which is represented by the preceding node. As equivalent heads must not be evaluated multiple times, the tree is unbalanced.

\section{Evaluation}
\label{sec:evaluation}

The purpose of the experimental evaluation was to demonstrate the applicability of the proposed SeCo algorithm despite the exponentially large search space. We did not expect any significant improvements in predictive performance since no enhancements in that respect
were made to the original algorithm as proposed in \cite{menc15}.

\shortsubsection{Experimental setup}
We compared our multi-label head algorithm to its single-label head counterpart and also to the binary relevance method on 8 different data sets.\footnote{\emph{scene} (6, 1.06), \emph{emotions} (6, 1.87), \emph{flags} (7, 3.39), \emph{yeast} (14, 4.24), \emph{birds} (19, 1.01), \emph{genbase} (27, 1.25), \emph{medical} (45, 1.24), \emph{cal500} (174, 26.15), with respective number of labels and cardinality, from \url{http://mulan.sf.net}. Source code and results are available at \linebreak\url{https://github.com/keelm/SeCo-MLC}.} Following \cite{menc15},
we used Hamming accuracy, subset accuracy (only for multi-label heads), micro-averaged precision and F-measure (with $\beta = 0.5$) on partial predictions for candidate rule selection and also allowed negative assignments $\hat{y}_{i} = 0$ in the heads.

\shortsubsection{Predictive performance}
Due to the space limitations, we limit ourselves to the results of the statistical tests (following \cite{Multiple-Comparisons}). The null hypothesis of the Friedman test ($\alpha=0.05$, $N=8$, $k=10$) that all algorithms have the same predictive quality could not be rejected for many of the evaluation measures, such as subset accuracy and micro- and macro-averaged F1. In the other cases, the Nemenyi post-hoc test was not able to assess a statistical difference between the algorithms using the same heuristic.

\newcommand{\labelfont}[1]{\emph{#1}}
\newcommand{\labelfontinv}[1]{\ensuremath{\overline{\text{\labelfont{#1}}}}}
\begin{figure}[!t]
\begin{minipage}{0.33\textwidth}
  \centering
  \includegraphics[width=1.0\textwidth]{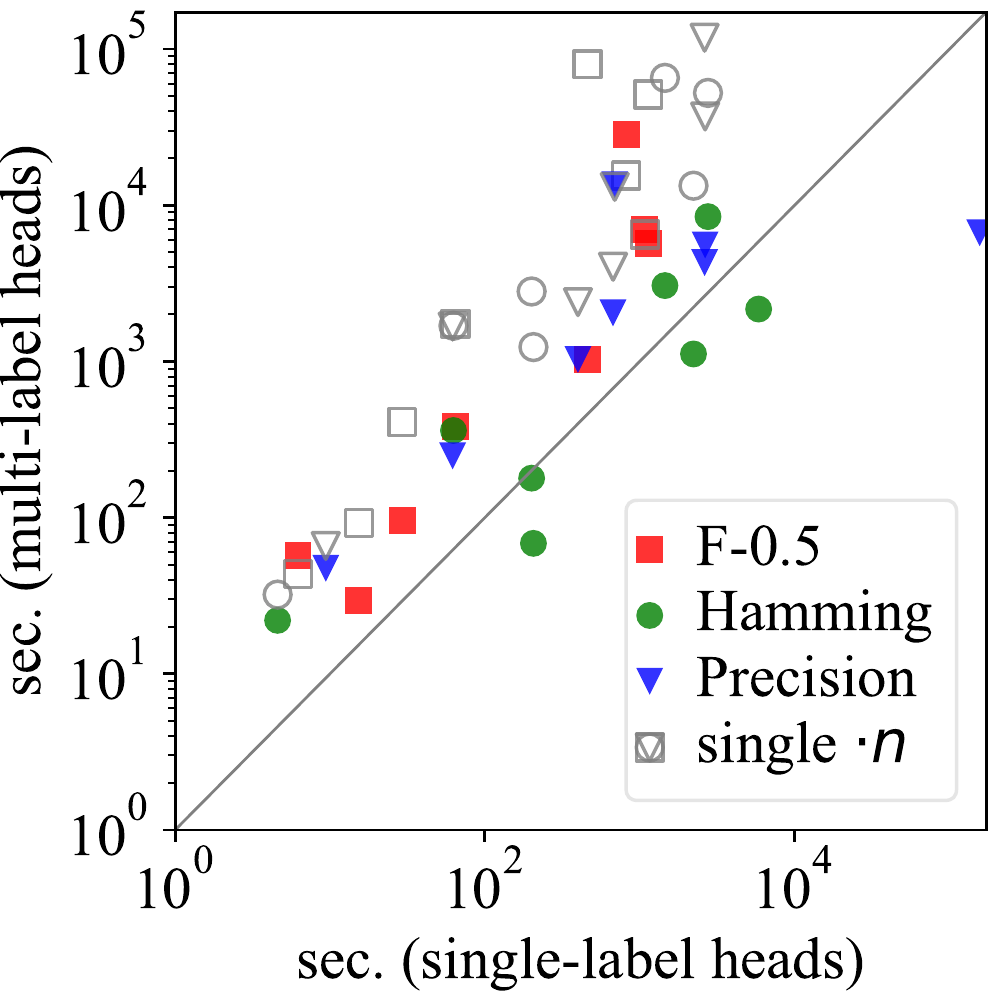}
  \captionof{figure}{Training times.}
  \label{fig:time_comp}
\end{minipage}
\hspace{0.01\textwidth}
\begin{minipage}{0.63\textwidth}
\resizebox{1.0\textwidth}{!}{
\begin{minipage}{1.02\textwidth}
\footnotesize
\noindent\rule{\textwidth}{0.6pt}
\labelfont{red}, \labelfont{green}, \labelfont{blue}, \labelfont{yellow}, \labelfont{white} $\gets$ colors$>$5, stripes$\leq$3 \ \ \  (65,0) \\ 
\labelfontinv{red}, \labelfont{green}, \labelfontinv{blue}, \labelfont{yellow}, \labelfont{white}, \labelfontinv{black}, \labelfontinv{orange} \\
\hspace*{\fill} $\gets$ animate, stripes$\leq$0, crosses$\leq$0 \ \ \  (11,0) \\
{\noindent\rule{1.0\textwidth}{0.3pt}}\\
\begin{tabular}{ l l | l r}
\labelfont{yellow} $\gets$ colors$>$4\ \  & (21,0) \ \ &\  \labelfont{green} $\gets$ text & \ \ (11,0)\\
\labelfont{red} $\gets$ \labelfont{yellow} & (21,0) &\ \labelfontinv{orange} $\gets$ saltires$<$1  \ \ & (1,0)\\
\labelfont{blue} $\gets$ colors$>$5 & (14,0) & \ \labelfontinv{black} $\gets$ area$<$11 &  (12,0) \\
\labelfont{white} $\gets$ \labelfont{blue} & (14,0)\\
\end{tabular}
\noindent\rule{\textwidth}{0.6pt}
\end{minipage}
}
\vspace{1.3em}
\captionof{figure}{Example of learned multi- and single-label head rule lists. $TP$ and $FP$ of respective rules are given in brackets.}
\label{fig:flags}
\end{minipage}
\end{figure}

\shortsubsection{Computational costs}
As expected, SeCo finds rules with a comparable predictive performance when searching for multi-label head rules. However, from the point of view of the proven properties of the evaluation measures, it was more interesting to demonstrate the usefulness of anti-monotonicity and decomposability regarding the computational efficiency. Fig.~\ref{fig:time_comp} shows the relation between the time spent for finding single- vs. multi-label head rules using the same heuristic and data set. The empty forms denote the single-label times multiplied by the number of labels in the data set. Note that full exploration of the labels space was already intractable for the smaller data sets on our system. We can observe that the costs for learning multi-label head rules are in the same order of magnitude despite effectively exploring the full label space for each candidate body.

\shortsubsection{Rule models}
When analyzing the characteristics of the models which have been learned by the proposed algorithm, it becomes apparent that more multi-label head rules are learned when using the precision metric, rather than one of the other metrics. This is due to the fact that precision only takes $TP$ and $FP$ into account. Therefore, the performance of such a rule depends exclusively on the examples it covers. When using another metric, where the performance also depends on uncovered examples, it is very likely that the performance of a rule slightly decreases when adding an additional label to its head. This causes single-label heads to be preferred. The inclusion of a factor which takes the head's size in account could resolve this bias and lead to heads with more labels.

Whether more labels in the head are more desirable or not highly depends on the data set at hand, the particular scenario and the preferences of the user, as generally do comprehensibility and interpretability of rules.
These issues cannot be solved by the proposed method, nor are in the scope of this work. However, the proposed extension of SeCo to multi-label head rules can lay the foundation to further improvements, gaining better control over the characteristics of the induced model and hence better adaption to the requirements of a particular use case.

The extended expressiveness of using multi-label head rules can be visualized by the following example. Consider the rules in Fig.~\ref{fig:flags}, learned on the data set \emph{flags} which maps characteristics of a flag and corresponding country to the colors appearing on the flag. The shown rules all cover the flag of the US Virgin Islands. Whereas in this case the single-label heads allow an easier visualization of the pairwise dependencies between characteristics/labels and labels, the multi-label head rules allow to represent more complex relationships and provide a more direct explanation of why the respective colors are predicted for the flag.

\section{Related work}
\label{sec:related}

So far, only a few approaches to multi-label rule learning can be found in the literature. Most of them are based on association rule (AR) discovery. Alternatively, a few approaches use evolutionary algorithms or classifier systems for evolving multi-label classification rules \cite{RuleBasedMLCwLCM,EvoMLCuARM,MLCRules}. Creating rules with several labels in the head is usually implemented as a post-processing step. For example, \cite{thabtah06MLassociative} and similarly \cite{MLCwARs}
induce single-label ARs which are merged to create multi-label rules. By using a separate-and-conquer approach the step of inducing descriptive but often redundant models of the data is omitted and it is directly tried to produce predictive rules \cite{menc15}.

Most of the approaches mentioned so far have in common that they are restricted to expressing a certain type of relationship since labels are only allowed as the consequent of a rule. Approaches that allow labels as antecedents of an implication are often restricted to global label dependencies, such as the approaches by \cite{jf:PL-08-WS-Park,LI-MLC,papagiannopoulou15deterministicrelations} that use the relationships discovered by AR mining on the label matrix for refining the predictions of multi-label classifiers.

The anti-monotonicity property is already well known from AR learning and subgroup discovery. For instance, it is used by the Apriori algorithm \cite{APriori} to prune searches for frequent item sets. \cite{robardet2016local} already used anti-monotonicity for efficiently mining subgroups in multi-label problems. However, in contrast to our work, they have not considered evaluation measures that are commonly used in MLC, but instead adapted metrics that are commonly used in subgroup discovery. We believe that the anti-monotonicity property must be assessed differently in a multi-label context. This is because AR learning neglects partial matches and labels that are not present in the heads (cf. Sec.~\ref{sec:relation}). In contrast, most MLC measures are much more sensitive in this respect. This is also demonstrated by the more restrictive property of decomposability which does not exist in common metrics for AR.

\section{Conclusions}
\label{sec:conclusions}

In this work, we formulated anti-monotonicity and decomposability criteria for multi-label rule learning and formally proved that several common multi-label evaluation measures meet these properties. Furthermore, we demonstrated how these results can be used to efficiently find rules with multi-label heads that are optimal with respect to commonly used multi-label evaluation functions. Our experiments showed that more work is needed to effectively combine such rules into a powerful rule-based theory.

\bibliographystyle{splncs04}
\footnotesize
\bibliography{bibliography}

\end{document}